\documentclass[letterpaper, 10 pt, conference]{ieeeconf}
\IEEEoverridecommandlockouts  
\usepackage{graphicx}
\usepackage[font=small]{caption}

\newif\ifarxiv

\arxivfalse

\IEEEoverridecommandlockouts

\usepackage{cite}
\usepackage{amsthm}  
\usepackage{amsmath,amssymb,amsfonts}
\usepackage{graphicx}
\usepackage{textcomp}
\usepackage[dvipsnames]{xcolor}         
\def\BibTeX{{\rm B\kern-.05em{\sc i\kern-.025em b}\kern-.08em
    T\kern-.1667em\lower.7ex\hbox{E}\kern-.125emX}}
\usepackage[hidelinks]{hyperref}       

\usepackage[utf8]{inputenc} 
\usepackage[T1]{fontenc}    
\usepackage[acronym]{glossaries}
\usepackage{arydshln}

\usepackage[american]{babel}

\usepackage{todonotes}
\usepackage{xspace}
\usepackage[linesnumbered,ruled,noend]{algorithm2e}

\usepackage{mathtools}
\usepackage{cite}
\usepackage{amsmath,amssymb,amsfonts}
\usepackage{bbm}

\usepackage{algpseudocode}
\usepackage{subcaption}
\usepackage{booktabs}
\usepackage{eso-pic}
\usepackage{multirow}
\usepackage{xcolor}

\newcommand{\unsafe}{\mathrm{u}}
\newcommand{\safe}{\mathrm{s}}
\newcommand{\initial}{\mathrm{0}}
\newcommand{\set}[1]{\{#1\}}

\newcommand{\pv}{\mathbf{v}}
\newcommand{\pw}{\mathbf{w}}
\newcommand{\px}{\mathbf{x}}

\newcommand{\po}{\mathbf{o}}

\newcommand{\pz}{\mathbf{z}}

\newcommand{\pX}{\mathbf{x}}

\renewcommand{\Pr}{\mathrm{Pr}}

\newcommand{\algebraic}{X = \big \{ x \in \reals^n \mid h_l(x) \geq 0 \; \forall l \in \set{1, \ldots, L} \big \}}

\newcommand{\obs}{X_{o}^{j}}
\newcommand{\obsU}{X_{o}}
\newcommand{\nobs}{N_{o}}
 
\newcommand{\reals}{\mathbb{R}}
\newcommand{\naturals}{\mathbb{N}}

\newcommand{\N}{\mathcal{N}}

\newcommand{\low}[1]{\underline{#1}}

\newcommand{\hor}{H}
\newcommand{\RA}{B}
\newcommand{\Pfin}{P_{\safe}(X_\safe,X_0,\{\obsU(k)\}_{k=0}^{H},\hor)}
\newcommand{\Ps}{P_{\safe}(X_\safe,X_0,\hor)}

\newtheorem{lemma}{Lemma}
\newtheorem{definition}{Definition}
\newtheorem{problem}{Problem}
\newtheorem{assumption}{Assumption}
\newtheorem{proposition}{Proposition}
\newtheorem{corollary}{Corollary}
\newtheorem{remark}{Remark}
\newtheorem{theorem}{Theorem}

\usepackage{fancyhdr}

\begin{document}

\title{\LARGE \bf
Stochastic Barrier Certificates
in the Presence of Dynamic Obstacles
}

\author{Rayan Mazouz$^{1}$, Luca Laurenti$^{2}$, Morteza Lahijanian$^{1}$
\thanks{$^{1}$ Authors are with the Dept. of Aerospace Eng. Sciences at University of Colorado Boulder, USA, 
       \texttt{\{firstname.lastname\}@colorado.edu}}%
\thanks{$^{2}$Author is with Delft University of Technology, The Netherlands,
\texttt{\{l.laurenti\}@tudelft.nl}}
}

\maketitle

\newif\ifarxiv

\arxivtrue

\ifarxiv
    \AddToShipoutPictureBG*{%
      \AtPageUpperLeft{%
        \hspace{12.5cm}%
        \raisebox{-1.5cm}{%
          \makebox[0pt][r]{To appear in the 65th IEEE Conference on Decision and Control (CDC 2026)}}}}
\fi

\begin{abstract}
    In this work, we focus on safety verification of stochastic systems in dynamic environments. We introduce both time-invariant and time-varying barrier certificates for discrete-time, continuous-space systems subject to uncertainty, which provide certified lower bounds on the probability of remaining within a safe set over a finite horizon. These certificates explicitly account for time-varying unsafe regions induced by obstacle dynamics. By leveraging Bellman's optimality perspective, the time-varying formulation directly captures temporal structure and yields less conservative bounds than state-of-the-art approaches.  By restricting certificates to polynomial functions,  we show that time-varying barrier synthesis can be formulated as a convex sum-of-squares program, enabling tractable optimization. Empirical evaluations on nonlinear systems with dynamic obstacles show that time-varying certificates consistently achieve tight guarantees, demonstrating improved accuracy and scalability over state-of-the-art methods.
\end{abstract}

\section{Introduction}

Complex dynamical systems operating under uncertainty make \emph{safety guarantees} a central challenge in cyber-physical systems\cite{guiochet2017safety}. In many real-world settings, these systems operate in environments that evolve over time, such as the presence of moving obstacles or time-varying constraints. Providing rigorous guarantees of system behavior in these settings requires reasoning over continuous state spaces, nonlinear stochastic dynamics, and temporally evolving safety specifications. Certificate-based methods, in particular stochastic barrier functions (SBFs) \cite{santoyo2021barrier, jagtap2020formal, mazouz2022safety}, have emerged as a powerful tool for certifying probabilistic safety of stochastic systems. However, existing formulations largely assume \emph{static} environments and rely on a single certificate enforced uniformly across time, limiting their ability to capture the induced temporal structure. In this work, we address these challenges by developing a \emph{time-varying} certificate-based framework that provides formal guarantees for stochastic systems evolving in dynamic environments. 

Existing certificate-based methods for finite-horizon safety guarantees of stochastic systems typically rely on a single barrier certificate enforced uniformly across all time steps \cite{yaghoubi2020risk, 
clark2021control, mazouz2024piecewise,
 xue2025finite}. This generally introduces conservatism by imposing the same certificate conditions regardless of the remaining time horizon or the evolution of the state distribution.
 Further, the time-dependent structure of the safety constraints, particularly in the presence of dynamic obstacles, are not directly captured. Recent work introduces interpolation-based barrier certificates to incorporate temporal variation \cite{oumer2025k}; however, the resulting guarantees are constructed  
 via Markov's and Boole's inequalities, which introduce additional conservatism in the computed probability of safety. Further, due to bilinear terms in the martingale condition, the formulation does not lend itself to a convex functional optimization problem. Hence, such formulations that incorporate temporal structure in the certificate construction remain limited in their scalability and ability to provide tight guarantees. 

In this paper, we focus on time-varying SBF certificates for discrete-time nonlinear stochastic systems with continuous spaces. 
First, we introduce a time-invariant approach that incorporates dynamic obstacles in an augmented space.  Then, we introduce a framework grounded in Bellman's optimality principle, i.e., dynamic programming (DP), which captures the backward reachability structure of safety and yields temporal
guarantees over finite horizons. This perspective naturally motivates varying certificates that evolve over time, thereby providing tighter bounds on safety compared to the time-invariant approach, and other state-of-the-art methods \cite{oumer2025k}. Moreover, by restricting the dynamics to the space of polynomial functions in the experiments, the proposed conditions admit a tractable sum-of-squares (SOS) relaxation \cite{legat2017sos, mazouz2026stochasticbarrier}, enabling numerically stable synthesis \cite{wajid2025successive, panja2025correct} via semi-definite programming \cite{vandenberghe1996semidefinite}. We evaluate the proposed framework on a range of nonlinear systems with dynamic obstacles, confirming the theoretical guarantees and demonstrating improved probability bounds and scalability compared to state-of-the-art formulations.

In short, the main contributions of this paper are:
\begin{itemize}
    \item A SBF certificate formulation encompassing both \emph{time-invariant} and \emph{time-varying} constructions for discrete-time stochastic systems over \emph{finite}-horizon settings, providing certified lower bounds on safety probabilities while accounting for \emph{dynamic obstacles}.
    \item A dynamic programming-based characterization of time-varying certificates, 
    establishing reduced conservatism in the resulting probability bound.
    \item Validation of the theoretical guarantees, along with benchmarking comparisons that highlight the advantages of the proposed time-varying formulation over existing certificate approaches.
\end{itemize}

\paragraph*{Related Work} 

Reasoning about safety in stochastic systems with dynamic environments has been addressed from multiple perspectives, including motion planning, control, and formal verification. Approaches for safe motion planning with dynamic obstacles \cite{aoude2013probabilistically, castillo2020real} typically rely on trajectory optimization or chance-constrained formulations, enabling (real-time) avoidance under uncertainty. Similarly, multi-agent and multi-robot frameworks \cite{theurkauf2025multi} focus on coordinating interactions in dynamic settings. In addition to the inability to provide a safety probability over the entire horizon, these methods are limited to linear dynamics with Gaussian noise. 
Our SBF framework provides safety probability bound with a certificate over the trajectory and extends to nonlinear dynamics.
Furthermore, existing SBF methods are typically conservative in practice, as the imposed functional conditions restrict the set of admissible certificates. Extensions such as piecewise SBFs~\cite{mazouz2024piecewise, mazouz2025piecewise} reduce this conservatism, but they remain confined to time-invariant formulations.  This paper relaxes these conditions and the resulting conservatism by enabling time-varying certificates.

Other approaches rely on learning-based or data-driven approximations of the certificate \cite{binny2026safe, salamati2021data, mazouz2024data, tayal2024learning, kong2026training}; however, such methods remain difficult to verify against hard safety constraints \cite{binny2026safe, tayal2024learning} or do not explicitly capture the temporal evolution of the certificate \cite{salamati2021data, mazouz2024data, kong2026training}. Overall, existing methods either lack formal certificate guarantees in dynamic environments or fail to fully exploit temporal structure in a tractable manner. In contrast, this work develops a DP-based formulation of time-varying stochastic barrier certificates that directly encodes temporal evolution that yield tight bounds.
\section{Problem Formulation}


We consider a discrete-time stochastic system described by the stochastic difference equation
\begin{equation} 
    \label{eq:system}
    \px_{k+1} = f(\px_{k}, \pw_k),
\end{equation}
where $k \in \mathbb{N}_0$ is the time step, $\px_k \in \mathbb{R}^n$ is the state at time step $k$, $f:\mathbb{R}^n \times \mathbb{R}^w \to \mathbb{R}^n$ is a Lipschitz continuous vector field, 
and $\pw_k \in \mathbb{R}^w$ is an i.i.d. random variable distributed according to $p_\pw$.
Intuitively, $\px_k$ represents a general stochastic process whose evolution may be governed by nonlinear dynamics and non-additive noise.

Given a state $x \in \mathbb{R}^n$ and a Borel-measurable set $X$, we define the transition kernel of System~\eqref{eq:system} as
\begin{equation}
\label{eq:transition_kernel}
    T(X\mid x):= \int_{\mathbb{R}^{\mathrm{w}}} \mathbf{1}_X
(f(x,w))p_{\mathbf{w}}(dw),
\end{equation}
where $\mathbf{1}_X(x) = 1$ if $x \in X$ and 0 otherwise.
It follows that, from an initial state $x_0 \in \mathbb{R}^n$, kernel $T$ defines a unique probability measure $\mathrm{Pr}^{x_0}$ over trajectories of System~\eqref{eq:system}~\cite{bertsekas2004stochastic}.


We consider the setting where System~\eqref{eq:system} shares the space with $\nobs \in \naturals_{\geq 1}$ dynamic obstacles (unsafe sets), e.g., other agents. 
Each obstacle $j \in \{1, \hdots, \nobs \}$ is a compact 
set, whose geometry at time $k$ is parameterized by configuration $o_k^j \in \reals^{p_j}$, e.g., configuration defined by center and orientation.
Configuration $o_k^j$ evolves according to
\begin{equation}
    \label{eq:obstacle}
    o^{j}_{k+1} = g^{j}(o_k^{j}),
\end{equation}
where $g^{j}: \reals^{p_j} \to \reals^{p_j}$ is a 
Lipschitz continuous map. 
Let $X^j_o: \reals^{p_j} \to \reals^n$ be the function that maps the obstacle configuration $o^j \in \reals^{p_j}$ to the set of points $X^j_o(o^j_k) \subset \reals^n$ that it occupies in the state space.
A collision occurs between System~\eqref{eq:system} and obstacle $j$ at time $k$ if $\px_k \in \obs(o^j_k)$.

In this work, we focus on quantifying \emph{probabilistic safety}, which, together with its dual notion of \emph{probabilistic reachability} \cite{laurenti2025unifying}, is commonly used to assess safety properties of stochastic systems. The following definition formalizes probabilistic safety in the presence of dynamic obstacles.

\begin{definition}[Probabilistic Safety]
    \label{def: probabilistic safety}
    Let $X_\safe \subset \reals^{n}$ be the safe set, $X_0 \subseteq X_\safe$ the initial set,
    $\obsU(k) = \cup_{j=1}^{\nobs} X^j_o(o^j_k)$ 
    the obstacle set at time $k$,
    and $H \in \mathbb{N}$ the horizon.
    Then, probabilistic safety in the presence of dynamic obstacles is
    \ifarxiv
    defined as
    \begin{multline*}
        \Pfin
        = \\
        \inf_{\substack{x_0 \in X_0}}
        \Pr^{x_0} \Big[
        \forall k \in \naturals_{\geq0}^{\leq H},\
        \pX_k \in 
        X_\safe 
        \setminus
        X_{o}(k)
        \Big].
    \end{multline*}
    \else
        \begin{multline*}
        \Pfin
        = \\
        \inf_{\substack{x_0 \in X_0}}
        \Pr^{x_0} \Big[
        \forall k \in \naturals_{\geq0}^{\leq H},\
        \pX_k \in 
        X_\safe 
        \setminus
        X_{o}(k)
        \Big].
    \end{multline*}
    \fi
    \ifarxiv
    In the absence of dynamic obstacles, the notation simplifies to $\Ps$.
    \fi
\end{definition}

We are interested in the theoretical and computational aspects of obtaining $P_\safe$. On the theoretical side, we develop a general framework for time-varying SBFs. 
On the computational side, we focus on tractable synthesis via SOS optimization, for which we impose the following assumption.
\begin{assumption}
    \label{asp:1}
    The system dynamics $f$ and obstacle maps $g^j$ are polynomial functions. All relevant sets  are semi-algebraic, i.e., they admit descriptions of the form $X = \{ x \in \mathbb{R}^n \mid h(x) \ge 0 \},$ where $h$ is a vector of polynomials. We further assume that distribution $p_\pw$ admits a closed-form moment generating function.
\end{assumption}
The problem we consider in this work is as follows.

\begin{problem}[Safety Verification]
\label{Prob:Verification}

    Consider stochastic System~\eqref{eq:system} on a compact safe set $X_\safe \subset \reals^{n}$ with dynamic obstacles $X_o^j(o^j_k)$, $j=1,\ldots,\nobs$ evolving according to~\eqref{eq:obstacle}. Given 
    \ifarxiv
    Assumption~\ref{asp:1}, 
    \fi
    initial set $X_0 \subseteq X_\safe$, and obstacle initial configurations $o_0^j \in \reals^{p_j}$, horizon $\hor \in \mathbb{N}$, and threshold $p_\safe \in [0,1]$, verify that the system stays in $X_\safe$ and avoids all obstacles over $\hor$ steps with probability at least $p_\safe$.
    i.e.,
    \begin{equation*}
        \Pfin \ge p_\safe .
    \end{equation*}
\end{problem}

Problem~\ref{Prob:Verification} aims to compute the probability that the system remains within a prescribed safe set while avoiding collisions with dynamic obstacles. This computation is particularly challenging because System~\eqref{eq:system} is stochastic and the obstacles are dynamic. Further, the functions $f$ and $g^j$ may be nonlinear. These features require propagating the state distribution through nonlinear stochastic dynamics while enforcing time-varying safety constraints at each time step, making finite-horizon safety probabilities difficult to characterize.
A common approach is to employ barrier certificates; however, the presence of time-varying obstacles renders their construction and analysis nontrivial.

\paragraph*{Approach overview}
%
Our approach to Problem~\ref{Prob:Verification} is based on \emph{safety certificates}, which serve as value functions that certify a lower bound on the probability of remaining safe over a finite horizon. We consider two formulations: one that lifts the problem to a time-invariant setting, enabling the use of existing barrier certificate constructions but with limited scalability due to the increased problem dimension. The second formulation directly constructs \emph{time-varying} certificates that explicitly capture dynamic obstacles, yielding tighter bounds and improved scalability. 
As part of this approach, we introduce a time-varying formulation for stochastic barrier functions.

\section{Preliminaries}
\label{sec:SOS}

Our approach to generate certificates is grounded in SBFs and polynomial functional analysis. 

\subsection{Stochastic Barrier Functions}
Consider {System}~\eqref{eq:system} with safe set $X_{\mathrm{s}} \subset \mathbb{R}^n$, initial set $X_0 \subseteq X_{\mathrm{s}}$, and unsafe set $X_{\mathrm{u}} = \mathbb{R}^n \setminus X_{\mathrm{s}}$. A function $B : \reals^n \to \mathbb{R}$ is called a SBF if there exist constants $\alpha, \beta \geq 0$,
\ifarxiv
such that
\fi
\begin{subequations}
    \label{eq:time-invariance SBF} 
    \begin{align}
        &B(x) \geq 0 \qquad &&\forall x\in \reals^n\label{eq:barrier_ss},\\
        &B(x) \geq 1 \qquad &&\forall x\in X_\mathrm{u}\label{eq:barrier_unsafe},\\
        &B(x) \leq \alpha \qquad &&\forall x\in X_0\label{eq:barrier_initial},\\
        &\mathbb{E}[B(f(x, \pw)) \; {|} \; {x}] \leq  B(x) + \beta \qquad && \forall x\in X_\mathrm{s} \label{eq:barrier_expectation},
    \end{align}
\end{subequations}
where the expectation is taken with respect to the transition kernel $T(\cdot \mid x)$ in~\eqref{eq:transition_kernel}.
Given a SBF $B$, a lower bound on the safety probability of {System}~\eqref{eq:system} is characterized as follows.
\begin{theorem}[\!{\cite{laurenti2025unifying}}]
\label{th:barrier}
If there exist function $B$ and scalars $\alpha,\beta$ that satisfy {conditions}~\eqref{eq:barrier_ss}-\eqref{eq:barrier_expectation}, then for a horizon $\hor\in \mathbb{N}$, it follows that
   \begin{align}
        \Ps
        \geq
        1-
        {\min  \{ 
        ( \alpha + \beta \hor), 1 \} }.\label{eq:probabilityBarrierFunctions}
    \end{align}
\end{theorem}
Equation~\eqref{eq:probabilityBarrierFunctions} provides a lower bound on the probability of remaining in the safe set $X_s$.  Note that this formulation holds only for static $X_\safe, X_\unsafe$. In our setting, $X_\unsafe$ is dynamic.


\subsection{Non-negative Polynomials}

The synthesis of SBFs can be posed as a functional optimization problem. To enable efficient solution via convex programming, we employ SOS polynomials, providing a sufficient condition for polynomial non-negativity~\cite{HilbertUeberDD}. 

\begin{definition}[SOS Polynomial]
\label{def:sosp1}
A multivariate polynomial $\lambda(x)$ in the variable $x \in \mathbb{R}^{n}$ is called a SOS polynomial if there exist polynomials $\lambda_i(x)$, for $i = 1, \ldots, r$ and some $r \in \naturals$, such that
$\lambda(x) = \sum_{i=1}^{r} \lambda_i^2(x).$
If $\lambda(x)$ is SOS, then it satisfies $\lambda(x) \geq 0$ for all $x \in \mathbb{R}^{n}$. The collection of all SOS polynomials is denoted by $\Lambda$.
\end{definition}

Now, let $\algebraic$ be a compact basic closed semi-algebraic set described by the intersection of $L$ polynomial inequality constraints. Proposition \ref{prop:sosp1} states a sufficient condition to guarantee non-negativity over such a set.
\begin{proposition}
\label{prop:sosp1}
Let $X \subset \reals^n$ be a compact basic semi-algebraic set given by
$\algebraic,$
where each $h_l(x)$ is a polynomial. A polynomial $\gamma(x)$ is non-negative on $X$ if there exist $\lambda_0(x), \lambda_l(x) \in \Lambda$ such that
$\gamma(x) = \lambda_0(x) + \sum_{l=1}^{L} \lambda_l(x) h_l(x).$
\end{proposition}

\begin{corollary}
\label{cor:sosp1}
Let $\gamma(x)$ be a non-negative on $\algebraic$, and $h(x)$ be an $L$-dimensional vector whose $l$-th entry is $h_l(x)$. Further, let $\mathcal{L}$ be an $L$-dimensional vector of SOS polynomials, meaning that its $l$-th component is $\lambda_l(x) \in \Lambda$. Then, it follows that $\gamma(x) - \mathcal{L}^{T} h(x) \in \Lambda$.
\end{corollary}

\section{Centralized Approach: Time-Invariant SBF}

We first formulate a centralized approach to Problem~\ref{Prob:Verification} that allows the adoption of (time-invariant) SBF in \eqref{eq:time-invariance SBF}.
%
%
To this end, we compose System~\eqref{eq:system} with the dynamic obstacles in one \emph{meta} system.  This is done by stacking up system state~$\px_k$ with obstacles configurations $o^j_k$, i.e., composed (meta) state: 
\ifarxiv
\begin{align*}
\pz_k := (\px_k, o_k^1, \ldots, o_k^{\nobs}) \in \reals^{n} \times \bigtimes_{j=1}^{\nobs} \reals^{p_j} =: Z.
\end{align*}
\else
\begin{align*}
\pz_k := (\px_k, o_k^1, \ldots, o_k^{\nobs}) \in \reals^{n} \times \bigtimes_{j=1}^{\nobs} \reals^{p_j} =: Z.
\end{align*}
\fi

Hence, the dynamics of the composed system are
\ifarxiv
given by
\begin{equation}
\label{eq:joint_dynamics}
    \pz_{k+1} = f_z(\pz_k,\pw_k),
\end{equation}
\else
$\pz_{k+1} = f_z(\pz_k,\pw_k),$
\fi
where $f_z(\cdot,\cdot) = \Big(
f(\cdot,\cdot),
g^{1}(\cdot),
\ldots,
g^{\nobs}(\cdot) 
\Big).$
This meta system evolves stochastically due to stochastic $\px_k$ component, even though the obstacle components $o^j_k$
\ifarxiv
evolve deterministically (according to \eqref{eq:obstacle}).
\else
evolve deterministically.
\fi
Hence, by lifting relevant sets to the composed space, we can use SBF for safety analysis
\ifarxiv
    of this system.
\fi

\begin{definition}[Sets of interests in the composed space]
    \label{def:sets}
    The obstacle set $Z_o$, safe set $Z_\safe$, and initial set $Z_0$ of the composed (meta) system 
    \ifarxiv
    in \eqref{eq:joint_dynamics} 
    \fi
    are defined as
    \ifarxiv
    \begin{align*}
        &Z_o := \{(x,o^1,\ldots,o^{\nobs}) \in Z \mid \exists j \in \naturals_{\geq 1}^{\leq \nobs} \text{ s.t. } x \in X^j_o(o^j)\}.\\
        &Z_\safe := \{ (x,o^1,\ldots,o^{\nobs}) \in Z\setminus Z_o \mid x \in X_\safe\},\\
        &Z_0 := \{ (x, o_0^1, \dots, o_0^{\nobs}) \in Z \mid x \in X_0  \}.
    \end{align*}
    \else
    \begin{align*}
        &Z_o := \{(x,o^1,\ldots,o^{\nobs}) \in Z \mid \exists j \in \naturals_{\geq 1}^{\leq \nobs} \text{ s.t. } x \in X^j_o(o^j)\}.\\
        &Z_\safe := \{ (x,o^1,\ldots,o^{\nobs}) \in Z\setminus Z_o \mid x \in X_\safe\},\\
        &Z_0 := \{ (x, o_0^1, \dots, o_0^{\nobs}) \in Z \mid x \in X_0  \}.
    \end{align*}
    \fi
\end{definition}


\begin{figure}[t!]
    \centering
\begin{tikzpicture}[scale=.90]

\draw[->] (-0.2,0) -- (5,0) node[right] {$x$};
\draw[->] (0,-0.2) -- (0,4) node[above] {$o$};

\definecolor{trajpurple}{RGB}{150,120,200}
\definecolor{basegreen}{RGB}{140,200,160}

\begin{scope}[rotate around={30:(2.5,2)}]
    \fill[gray!25] (2.5,2) ellipse (1.2 and 0.6);
    \draw[black, thick] (2.5,2) ellipse (1.2 and 0.6);
    \node at (2.5,2) {\small Obstacle};
\end{scope}

\draw[trajpurple, dashed]
    (0.5,0.5)
    .. controls (1.2,0.6) and (1.8,0.8) ..
    (2.2,0.9)
    .. controls (2.8,1.0) and (3.4,1.1) ..
    (3.8,1.5)
    .. controls (4.2,2.0) and (4.4,2.6) ..
    (4.5,3.2);

\foreach \x/\y/\rx/\ry/\angle in {
    0.5/0.5/0.18/0.10/10,
    1.4/0.7/0.22/0.12/15,
    2.2/0.9/0.26/0.14/20,
    3.2/1.2/0.30/0.16/25,
    4.0/2.0/0.34/0.18/30,
    4.5/3.2/0.38/0.20/35
}{
    \pgfmathsetmacro{\ra}{\rx}
    \pgfmathsetmacro{\rb}{\ry}
    \pgfmathsetmacro{\intensity}{min(80, max(20, 200*(\ra+\rb)))}

    \begin{scope}[rotate around={\angle:(\x,\y)}]
        \fill[basegreen!\intensity] (\x,\y) ellipse ({\ra} and {\rb});
        \draw[trajpurple, thick] (\x,\y) ellipse ({\ra} and {\rb});
    \end{scope}
}

\fill[trajpurple] (0.5,0.5) circle (1.2pt);
\fill[trajpurple] (4.5,3.2) circle (1.2pt);

\node[trajpurple] at (0.7,0.2) {$k=1$};
\node[trajpurple] at (4.5,3.8) {$k = H$};

\end{tikzpicture}
\caption{Propagation of state distributions avoiding a static obstacle in the augmented space $Z$.}
\vspace{-2mm}
\label{fig:distribution_propagation}
\end{figure}
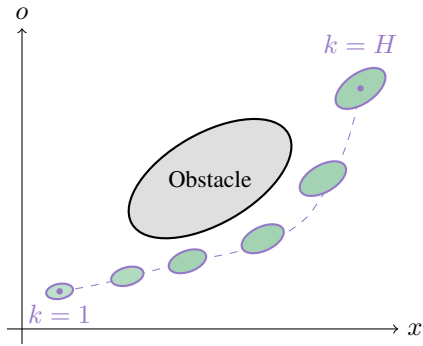

\ifarxiv
Note that, with this definition of $Z_o$ in the composed space, the obstacle set becomes static, formalized by. This is formalized in the following lemma.
\else
With this definition of $Z_o$ in the composed space, the set becomes static, which the following lemma formalizes.
\fi

\begin{lemma}
    \label{lemma: obstacle set in Z}
    Consider System~\eqref{eq:system} with state $\px_k$ and dynamic obstacle configurations $o^j_k$, $j=\{1,\ldots, \nobs\}$ that evolve according to \eqref{eq:obstacle}, and the composed meta system
    \ifarxiv
    ~\eqref{eq:joint_dynamics} 
    \fi
    with state $\pz_k = (\px_k, o^1_k, \ldots, o^{\nobs}_k)$ and obstacle set $Z_o$ in Definition~\ref{def:sets}.  It holds that,
    \ifarxiv
    for all $k \in \naturals$,
    \begin{align}
        \px_k \in \bigcup_{j=1}^{\nobs} X^j_o(o^j_k) \; \iff \; \pz_k \in Z_o.
    \end{align}
    \else
    $\forall k \in \naturals$,
    $ \px_k \in \bigcup_{j=1}^{\nobs} X^j_o(o^j_k)$ iff $\pz_k \in Z_o.$
    \fi
\end{lemma}

The intuition is that the time dependence of the obstacles is absorbed into the augmented state $z$. In the composed space $Z$, the obstacle configurations are explicitly represented as part of the state, so collision checking depends only on $z$, not on time. As a result, the obstacle set $Z_o$ becomes time-invariant, while the obstacle dynamics are incorporated into the composed meta system. Fig.~\ref{fig:distribution_propagation} visualizes this intuition.

Using this augmented representation, the following theorem formalizes the problem of synthesizing a time-invariant SBF for System~\eqref{eq:system} with dynamic obstacles.  

\begin{theorem}[Time-Invariant Certificates with Dynamic Obstacles]
    \label{th:time-invariant}
    Consider stochastic System~\eqref{eq:system} with dynamic obstacles as described in Problem~\ref{Prob:Verification}, the composed meta system 
    \ifarxiv
    in~\eqref{eq:joint_dynamics}, 
    \fi
    and the obstacle set $Z_o$, safe set $Z_\safe$, and initial set $Z_0$ defined in Definition~\ref{def:sets}. 
    If there exist a continuous function $\RA : Z  \to \reals_{\geq 0}$ and scalars $\alpha, \beta \geq 0$, such that
    \begin{subequations}
        \begin{align}
            &\RA(z) \geq 0 && \forall z \in Z, \\
            &\RA(z) \geq 1 && \forall z \in Z_o,\\
            &\RA(z) \leq \alpha && \forall z \in Z_0, \\
            &\mathbb{E}[\RA(f_z(z,\pw)) \mid z ] \leq \RA(z) + \beta
            && \forall z \in Z_\safe,
        \end{align}
    \end{subequations}
    \ifarxiv
    then, it holds that 
       \begin{align}
            \Pfin 
            \geq
            1 - \left(\alpha + \beta\cdot H\right).
        \end{align}
    \else
    then,  
    $\Pfin 
            \geq
            1 - \left(\alpha + \beta H\right)$
    holds.
    \fi
\end{theorem}
\ifarxiv
\begin{proof}
    Per Theorem~\ref{th:barrier}, existence of $B$ guarantees $P_\safe({Z_\safe,Z_0},\hor) \geq 1 - \left(\alpha + \beta\cdot H\right)$ for the composed system.  Then, by Lemma~\ref{lemma: obstacle set in Z} and definition of $P_\safe$ in Def.~\ref{def: probabilistic safety}, we obtain $P_\safe({Z_\safe,Z_0},\hor) = \Pfin$.
\end{proof}
\else
    A proof is provided in the extended version~\cite{mazouz2026stochastic}. 
\fi


\begin{remark}
    The result of Theorem~\ref{th:time-invariant} extends straightforwardly to stochastic dynamic obstacles, i.e., $\po^j_{k+1} = g^j(\po^j_k,\pv_k)$, where $\pv_k$ is a random process.
\end{remark}

\begin{remark}
    \label{remark: SOS optimal for time-invariant}
    If sets $X_\safe, X_0, X^1_o(o^1), \ldots , X^{\nobs}_o(o^\nobs) \subset \reals^n$ are semi-algebraic, then sets $Z_\safe,Z_0,Z_o \subset Z$ are also semi-algebraic.  Hence, 
    \ifarxiv
        under Assumption~\ref{asp:1},
    \fi
    the certificate synthesis in Theorem~\ref{th:time-invariant} can be formulated as an SOS program
    as shown in~\cite{santoyo2021barrier, mazouz2022safety}.
\end{remark}

Theorem~\ref{th:time-invariant} solves Problem~\ref{Prob:Verification} by lifting the system to a higher-dimensional space. While this reduces the complexity of handling dynamic obstacles by converting them into a static setting, it introduces computational challenges, as solving the functional optimization problem for the certificate in high dimensions is often intractable. Next, we introduce an alternative approach that avoids this curse of dimensionality by incorporating the complexity into the SBF itself.

\section{Time-Varying SBF}


We introduce a \emph{time-varying} formulation for SBFs based on DP that applies to settings with dynamic obstacles.
Unlike interpolation-based formulations~\cite{oumer2025k}, our approach adopts a backward recursion. This naturally accommodates time-varying unsafe sets by using a separate SBF at each time step, yields tighter safety bounds as shown in the experiments, and admits a convex program.
Theorem~\ref{th:time-varying} provides the formulation of the time-varying \emph{bounded-horizon} certificate.
\begin{theorem}[Time-Varying Certificates with Dynamic Obstacles]
\label{th:time-varying}
    Consider System~\eqref{eq:system} 
    with initial set $X_0$, safe set $X_s$, and time-varying unsafe set $X_u(k) := (\mathbb{R}^n \setminus X_{\safe}) \cup \obsU(k)$ for $k\in \naturals_{\geq 0}$, where $\obsU(k)$ is defined in Def.~\ref{def: probabilistic safety}. For horizon $\hor \in \naturals_0$, a real-valued function $\RA : \mathbb{R}^n \times \naturals^{\leq \hor}_{\geq 0} \to \mathbb{R}_{\geq 0}$ that is continuous in its first argument is a \emph{time-varying barrier certificate} if, for some constants $\alpha, \beta_i \in \reals_{\geq 0}$, it satisfies
    \begin{subequations}
    \label{eq:certificate-alt-tv}
        \begin{align}
        \!\!\!\!\! &\RA(x, i) \geq 0 
        && \forall x \in \mathbb{R}^n,\ \hspace{4mm} \forall i \in \naturals^{\leq \hor}_{\geq 0}, \label{eq:nonneg-alt-tv} \\
        &\RA(x, i) \geq 1 
        && \forall x \in X_\unsafe(i),\ \forall i \in \naturals^{\leq \hor}_{\geq 0}, \label{eq:unsafe-alt-tv} \\
        &\RA(x, H) \leq \alpha 
        && \forall x \in X_\initial, \label{eq:initial-alt-tv} \\
        &\mathbb{E}[\RA(f(x,\pw), i-1)\mid x] 
        && \nonumber \\
        &\qquad \leq \RA(x, i) + \beta_i 
        && \forall x \in X_\safe,\ \hspace{4mm} \forall i \in \naturals^{\leq \hor}_{\geq 1}, \label{eq:exp-alt-tv}
        \end{align}
    \end{subequations}
    Then, the safety probability is lower-bounded by
    \begin{align}
    \label{eq:tv-bound}
        \Pfin \geq 1 - \left(\alpha  + \sum_{i=1}^{ \hor} \beta_i
        \right ).
    \end{align}  
\end{theorem}
\ifarxiv
\begin{proof}
    Let $V_k(x)$ denote the probability that System~\eqref{eq:system} starting from $x$ reaches the unsafe set $X_u$ with $k$ time steps. Then $V_k(x)$ satisfies the following DP \cite{laurenti2025unifying}
    \begin{align*}
    V_k(x) =
        \begin{cases}
        1 & \text{if } x \in X_u(k),\\
        0 & \text{if } x \in X_s \text{ and } k = 0,\\
        \mathbb{E}[\,V_{k-1}(x') \mid x\,] & \text{if } x \in X_s,~ k > 0.
        \end{cases}
    \end{align*}
    Then, to conclude, for all $x \in \mathbb{R}^n$ and $k \in \mathbb{N}^{\le \hor}_{\ge 0}$, it is enough to show that
    $V_k(x) \le B(x,k) + \sum_{i=1}^{k}\beta_i.$
We do that by induction. For $k = 0$, we have that $V_0(x)=1$ for all $x\in X_u(0)$ and $0$ otherwise. Consequently, since 
$B(x,0) \ge 0,$ for all $x\in\mathbb{R}^n$ it holds that
$V_0(x) \le B(x,0).$ 
Next, assume the claim holds for $k-1$, i.e.,
$V_{k-1}(x') \le B(x',k-1) + \sum_{i=1}^{k-1}\beta_i
\quad \forall x'\in\mathbb{R}^n.$
Then for $x \in X_s$, it holds that
\begin{align*}
    V_k(x)
    = \mathbb{E}[\,V_{k-1}(x') \mid x\,] 
    &\le \mathbb{E}\Big[\,B(x',k-1) + \sum_{i=1}^{k-1}\beta_i \,\Bigm|\, x\Big] \\
    &= \mathbb{E}[\,B(x',k-1)\mid x\,] + \sum_{i=1}^{k-1}\beta_i.
\end{align*}
which gives
$        \mathbb{E}[\,B(x',k-1)\mid x\,] \le B(x,k) + \beta_k.$
    Substituting this into the previous inequality gives
\begin{align*}
    V_k(x)
    \le B(x,k) + \beta_k + \sum_{i=1}^{k-1}\beta_i 
    = B(x,k) + \sum_{i=1}^{k}\beta_i.
\end{align*}
\end{proof}
\else
    A proof is provided in the extended version \cite{mazouz2026stochastic}. 
\fi
Intuitively, the time-varying barrier certificates in Theorem~\ref{th:time-varying} extend standard barrier certificates by allowing for each certificate to closely approximate the value function $V_k$ defined in the proof of Theorem~\ref{th:time-varying} at every time step. As we show in Section~\ref{sec:results}, this increased flexibility allows for tighter lower bounds on the probability of safety. 


We now establish the SOS formulation for the time-varying certificate under dynamic obstacles.
\begin{theorem}[Time-Varying SOS Certificate]  
\label{th:one-shot-certificate}
Let $\{\RA(x, i)\}_{i=0}^H$ be a sequence of polynomial functions. Further, consider semi-algebraic sets 
$X_{\safe} = \{ x \in \reals^n \mid h_s(x) \ge 0\}$, 
$X_\initial = \{ x \in \reals^n \mid h_\initial(x) \ge 0\}$, 
$X_{\unsafe} = \{ x \in \reals^n \mid h_u(x) \ge 0\}$, 
and $X_o^j(i) = \{ x \in \reals^n \mid h_{o,i}^j(x; o_i^j) \ge 0 \}$ for each $i \in \naturals^{\leq \hor}_{\geq 0}, j \in \{1,\dots,\nobs\}$, 
where $h_s$, $h_\initial$, $h_u$, and $h_{o,i}^j$ are vectors of polynomials. 
Let $\mathcal{L}_{s}(x)$, $\mathcal{L}_\initial(x)$, $\mathcal{L}_\unsafe(x)$, and $\mathcal{L}^j_{o,i}(x, o_i^j)$ be vectors of SOS polynomials of appropriate dimensions. Then, a \emph{time-varying barrier certificate} can be obtained by solving the SOS optimization problem:
    \begin{subequations}
          \label{pr:sos-1}
        \begin{align}
        & \!\!\! \min_{\substack{\alpha, (\beta_i)_{i=1}^{H}}} \!\!         \alpha + \sum_{i=1}^{ \hor} \beta_i \quad 
         \text{subject to:} \nonumber \\
        &\RA(x, i)\in \Lambda(x) && \hspace{-20mm} \forall i \in \naturals^{\leq \hor}_{\geq 0} \label{constraint1} \\
        & \RA(x, i) - 1 - \mathcal{L}_{\unsafe}(x)^{\top}h_{\unsafe}(x)  \in \Lambda(x) && \hspace{-20mm} \forall i \in \naturals^{\leq \hor}_{\geq 0}  \label{constraint2} \\
        & \RA(x, i) - 1 - \mathcal{L}^j_{o,i}(x, o^j_i)^{\top}h^j_{o,i}(x, o^j_i)  \in \Lambda(x, o^j_i) \nonumber \\ & \hspace{40mm} \forall i \in \naturals^{\leq \hor}_{\geq 0}, \forall j \in \naturals_{\geq 1}^{\leq \nobs} \!\!\!\!\!\!\!\!\! && \label{constraint2b} \\
        \!\!\!\!- & \RA(x, H) + \alpha - \mathcal{L}_{\initial}(x)^{\top}h_{\initial}(x) \in \Lambda(x) && \label{constraint3} \\
          & \RA(x, i) - \mathbb{E}[\RA(\px', i-1) \mid x ]  \!  + \beta_i -  
        + \nonumber \\
        & \hspace{2mm}   \mathcal{L}_{\safe}(x)^{\top}{h_{s}(x)} \mathcal{L}^j_{o,i}(x, o^j_i)^{\top}h^j_{o,i}(x, o^j_i) \in \Lambda(x, o^j_i)  
        \nonumber \\
        & \hspace{41mm} \forall i \in \naturals^{\leq \hor}_{\geq 1}, j \in \naturals_{\geq 1}^{\leq \nobs}
        \label{constraint5}    
        \end{align}
    \end{subequations}
which guarantees the probabilistic safety bound in~\eqref{eq:tv-bound}.
\end{theorem}

\ifarxiv
\begin{proof}
    It suffices to show that if $\RA$ satisfies Conditions~\eqref{eq:nonneg-alt-tv}-\eqref{eq:exp-alt-tv}, bound in~\eqref{eq:tv-bound} holds.
    Constraint~\eqref{constraint1} implies $\RA(x,i)\ge 0$, yielding~\eqref{eq:nonneg-alt-tv}. 
By the S-procedure, Constraints~\eqref{constraint2}--\eqref{constraint2b} imply $\RA(x,i)\ge 1$ on $X_\unsafe(i)$, giving~\eqref{eq:unsafe-alt-tv}. 
Constraint~\eqref{constraint3} yields $\RA(x,H)\le \alpha$ on $X_\initial$, i.e.,~\eqref{eq:initial-alt-tv}. 
Finally, Constraint~\eqref{constraint5} enforces~\eqref{eq:exp-alt-tv} on $X_\safe$, completing the proof.
\end{proof}
\else
    A proof is provided in the extended version \cite{mazouz2026stochastic}. 
\fi
\begin{remark} 
    The framework can be extended to joint certificate and control synthesis, as shown in\cite{mazouz2026timevarying}. 
\end{remark}
The functional optimization problem resulting from Theorem~\ref{th:time-varying} introduces a larger number of optimization parameters to a time-invariant instance,
as it assigns a distinct barrier function $B(x,i)$ to each time step $i$. However, the parameterization does not incur the curse of dimensionality associated with the state augmentation in Theorem~\ref{th:time-invariant}. Instead, the time-indexed structures 
reduces the need for high expressive function classes that induce computational burden. As shown in our results, this formulation consistently yields significantly higher certified safety probabilities compared to time-invariant and interpolation-based approaches, while maintaining favorable scalability.

\section{Results}
\label{sec:results}

\begin{table*}
    \caption{Barrier benchmark results for \texttt{time-invariant}, \texttt{interpolation}, and \texttt{time-varying} approaches; the latter two use the same barrier degree $\texttt{deg}$. $\underline{P}_\safe$ is the obtained lower bound on $P_\safe$, and $\tau$ is the computation time in seconds. A time-out (\texttt{TO}) time of 1000s is included in the benchmarks.}
    \label{tab:barrier}
    \centering
    \vspace{-1mm}
    \scalebox{0.9}
    {
         \begin{tabular}{l|c|c cccc| c cc | cccc }
            \toprule
        \multirow{2}{*}{Model} & \multirow{2}{*}{$H$} & \multirow{2}{*}{} 
        & \multicolumn{4}{c|}{\texttt{Time-Invariant}} 
        & \multirow{2}{*}{}
        & \multicolumn{2}{c|}{\texttt{Interpolation\cite{oumer2025k}}}  
        & \multicolumn{4}{|c}{\texttt{Time-Varying} (ours)}
        \\ 
        & & \texttt{deg}
        & {$\alpha$} & {$ \beta \cdot \hor$}    & $\low{P}_\safe$ &$\tau(s)$ 
        &\texttt{deg} & $\low{P}_\safe$ & $\tau(s)$
        & {$\alpha$} & {$\sum \beta_i$}  & $\low{P}_\safe$   & $\tau(s) $
        \\
        \hline
        Unstable 1D  
        & 5 & 30 
        & $7.6e^{-2}$ & $2.4e^{-2}$ & 0.90 & 142.54
        & 4 & 0.92 & \textbf{0.07}
        & $1.6e^{-2}$ & $6.1e^{-5}$ & \textbf{0.98} & 0.11 \\
        & 10 & 30
        & $8.2e^{-2}$&  $4.8e^{-2}$& 0.87 & 144.25 
        & 4 & 0.74 & 0.15
        & $2.7e^{-2}$ & $1.2e^{-3}$ & \textbf{0.97} & \textbf{0.14} \\
        & 15 & 30
        & $7.3e^{-2}$ & $8.1e^{-2}$ & 0.85 & 147.27
        & 4 & 0.35 & 0.62
        & $4.3e^{-2}$ & $6.8e^{-4}$ & \textbf{0.96} & \textbf{0.34} \\
        & 20 & 30 & $5.4e^{-3}$ & $1.7e^{-1}$ & 0.82 & 149.34 & 6  & 0.84 & 76.19 & $8.5e^{-4}$ & 0.05 & \textbf{0.90} & \textbf{7.18} \\
        \hline
        Unstable 2D  & 5 & 24
        & $2.6e^{-1}$ & $6.9e^{-2}$ & 0.67 & 179.00
        & 4 & 0.84 & 0.33
        & $3.3e^{-2}$ & $1.1e^{-3}$ & \textbf{0.97} & \textbf{0.21} \\
        & 10 & 24
        & $5.1e^{-2}$ & $3.6e^{-2}$ & 0.45 & 177.23
        & 4 & 0.47 & 0.96
        & $5.4e^{-2}$ & $2.3e^{-3}$ & \textbf{0.94} & \textbf{0.39} \\
        & 15 & 24
        & $7.2e^{-2}$ & $3.2e^{-2}$ & 0.24 & 176.41
        & 4 & 0.00 & 2.07
        & $8.5e^{-2}$ & $7.5e^{-4}$ & \textbf{0.91} & \textbf{1.06} \\
        & 20 & 24
        & $7.9e^{-1}$ & $1.2e^{-2}$ & 0.19 & 185.42
        & 6 & 0.34 & 50.12
        & $1.1e^{-1}$ & $1.2e^{-3}$ & \textbf{0.89} & \textbf{24.17} \\
        \hline
        Oscillator 2D & 10 & 10
        & $7.4e^{-1}$ & $6.2e^{-2}$ & 0.22 & 7.42
        & 2 & 0.00 & 7.88
        & 0.60 & $5.2e^{-4}$ & \textbf{0.40} & \textbf{2.34} \\
        & 15 & 10
        & $7.5e^{-1}$ & $1.2e^{-1}$ & 0.13 & 8.01
        & 4 & 0.00 & 12.85
        & $5.4e^{-8}$ & $6.5e^{-7}$ & \textbf{0.99} & \textbf{5.53} \\
        & 20 & 10 & $7.2e^{-1}$ &  $1.6e^{-1}$ & 0.12 & \textbf{7.59} & 4 & 0.00 & 41.07 & $4.9e^{-8}$ & $9.9e^{-7}$ & \textbf{0.99} & 8.71  \\
        \midrule
        Volterra 2D & 5 & 10
        & $1.1e^{-2}$ & $3.6e^{-2}$ & 0.95 & 19.01
        & 4 & \textbf{0.99} & 0.89
        & $2.7e^{-6}$ & $1.4e^{-5}$ & \textbf{0.99} & \textbf{0.81} 
        \\
        & 10 & 10
        & $1.1e^{-2}$ & $7.2e^{-2}$ & 0.92  & 19.46
        & 4 & \textbf{0.99} & 2.26
        & $1.4e^{-6}$ & $1.9e^{-5}$ & \textbf{0.99} & \textbf{1.98} \\
        & 20 & 10
        &$2.1e^{-2}$ & $4.9e^{-1}$  & 0.49 & 19.49
        & 4 & \textbf{0.99} & 3.41
        & $1.1e^{-6}$ & $1.6e^{-5}$ & \textbf{0.99} & \textbf{2.87} \\
        \midrule
        Dubin's 4D & 5 & 4
        & 1.00 & $4.9e^{-6}$ & 0.00 & 63.03
        & 2 & - & \texttt{TO}
        & $1.5e^{-1}$ & $3.5e^{-4}$ & \textbf{0.85} & \textbf{345.21} \\
        & 10 & 4
        & 1.00 & $5.0e^{-6}$ & 0.00 & 65.86
        & 2 & - & \texttt{TO}
        & $1.5e^{-1}$ & $6.1e^{-4}$ & \textbf{0.85} & \textbf{721.28} \\
         & 10 & 6
        & 0.80 & $7.8e^{-2}$ & 0.12 & \textbf{742.53}
        & 4 & - & \texttt{TO}
        & $1.3e^{-1}$ & $7.4e^{-4}$ & \textbf{0.87} & 802.17 \\
        \bottomrule
        \end{tabular}
    }
    \vspace{-2mm}
\end{table*}

\begin{table*}[t!]
    \caption{Barrier benchmark results with dynamic obstacles for \texttt{time-invariant}, \texttt{interpolation}, and \texttt{time-varying} approaches; the latter two use the same barrier degree $\texttt{deg}$. $\underline{P}_\safe$ is the obtained lower bound $P_\safe$, and $\tau$ is the computation time in seconds. A time-out (\texttt{TO}) time of 1000s is included in the benchmarks.}
    \label{tab:barrier-dynamic-obtacles}
    \centering
    \scalebox{0.9}
    {
        \begin{tabular}{l|c|c cccc| c cc | cccc }
            \toprule
            \multirow{2}{*}{Model} & \multirow{2}{*}{$H$} & \multirow{2}{*}{} 
            & \multicolumn{4}{c|}{\texttt{Time-Invariant}} 
            & \multirow{2}{*}{}
            & \multicolumn{2}{c|}{\texttt{Interpolation\cite{oumer2025k}}}  
            & \multicolumn{4}{|c}{\texttt{Time-Varying} (ours)} 
            \\ 
            & & \texttt{deg}
            & {$\alpha$} & {$ \beta \cdot \hor$}    & $\low{P}_\safe$ &$\tau(s)$ 
            & \texttt{deg} & $\low{P}_\safe$ & $\tau(s)$
            & {$\alpha$} & {$\sum \beta_i$}  & $\low{P}_\safe$   & $\tau(s) $
            \\
            \midrule
           Unstable 1D & 5 & 24
        & $5.1e^{-2}$ & 0.11 & 0.84 & 192.16
        & 4 & 0.88 & 0.19
        & $2.1e^{-2}$ & $8.5e^{-5}$ & \textbf{0.94} & \textbf{0.13} \\
        
        \textit{+1D obstacle} & 10 & 24  & $6.7e^{-2}$ & 0.39 & 0.54 & 198.57 
        & 4 & 0.62 & \textbf{0.21}
        & $3.6e^{-2}$ & $2.4e^{-3}$ & \textbf{0.89} & {0.39} \\
            \midrule
            Oscillator 2D & 5 & 10 
            & $6.7e^{-1}$ & $4.6e^{-2}$ & 0.28 & 884.12
            & 2 & 0.38 & 289.17 
            & $3.5e^{-1}$ & $8.6e^{-5}$ & \textbf{0.65} & \textbf{5.49} \\
            
            \textit{+2D obstacle} & 10 & 10 
            & $6.9e^{-1}$ & $8.9e^{-2}$ & 0.22 & 900.12 
            & 4 & 0.22 & 786.22 
            & $1.9e^{-1}$ & $2.7e^{-6}$ & \textbf{0.81} & \textbf{27.58} \\
            
            & 15 & 10 
            & $7.4e^{-1}$ & $8.0e^{-2}$ & 0.18 & 912.37
            & 4 & - & \texttt{TO} 
            & $2.1e^{-1}$ & $3.1e^{-5}$ & \textbf{0.79} & \textbf{37.45} \\

            \hline
            Oscillator 2D & 5 & 6 
            & $6.9e^{-1}$ & $1.1e^{-2}$ & 0.20 & 914.38
            & 2 & - & \texttt{TO} 
            & $3.8e^{-1}$ & $2.2e^{-4}$ & \textbf{0.62} & \textbf{316.86} 
            \\
            
            \textit{2 $\times$ 2D obstacle} & 10 & 6 
            & $6.4e^{-1}$ & $2.3e^{-1}$ & 0.13 & 911.56 
            & 4 & - & \texttt{TO} 
            & $2.7e^{-1}$ & $2.0e^{-4}$ & \textbf{0.73} & \textbf{841.17} 
            \\
            
            & 15 & 6 
            & $8.9e^{-1}$ & $4.2e^{-2}$ & 0.07 & 953.88
            & 4 & - & \texttt{TO}  
            & $3.1e^{-1}$ & $4.3e^{-3}$ & \textbf{0.69} & \textbf{983.15} 
            \\

            \hline
            Quadrotor 4D & 5 & 4 
            & 1.00 & $7.6e^{-6}$ & 0.00 & 619.58  
            & 2 & - & \texttt{TO} 
            & $4.8e^{-2}$ & $2.9e^{-6}$ & \textbf{0.95} & \textbf{354.29} \\  
            
            \textit{+2D obstacle} & 10 & 4 
            & 1.00 & $3.5e^{-6}$ & 0.00 & 628.89 
            & 2 & - & \texttt{TO} 
            & $4.7e^{-2}$ & $4.6e^{-5}$ & \textbf{0.95} & \textbf{743.18} \\ 

            & 10 & 6
            & 0.68 & $1.1e^{-1}$ & 0.21 & 912.44
            & 4 & - & \texttt{TO}
            & $5.1e^{-2}$ & $6.8e^{-5}$ & \textbf{0.92} & \textbf{902.73} \\
                        
            \bottomrule
        \end{tabular}
    }
    \vspace{-5mm}
\end{table*}

In this section, we evaluate the proposed framework on a range of stochastic systems with dynamic obstacles. 
We compare the performance of the following approaches over a finite horizon $\hor \in \naturals_{\geq 0}$: (i) time-invariant, (ii) interpolation ~\cite{oumer2025k}, and (iii) time-varying certificates. 
All methods were implemented in Julia, leveraging \texttt{SumOfSquares.jl}~\cite{weisser2019polynomial} together with \texttt{JuMP}~\cite{dunning2017jump} to perform SOS optimization 
\ifarxiv
\footnote{Since the method in \cite{oumer2025k} does not admit an SOS formulation due to the bilinear term, as suggested in the paper, we obtain an SOS program by fixing the bilinear decision variable, and then iteratively optimize for it.}.
\fi
All experiments were conducted on a system featuring a 3.9 GHz 8-core processor and 128 GB of RAM. 

An overview of the results is reported in Tables~\ref{tab:barrier} and \ref{tab:barrier-dynamic-obtacles}. Table~\ref{tab:barrier} presents results for benchmark systems without obstacles, serving to evaluate the underlying frameworks, while Table~\ref{tab:barrier-dynamic-obtacles} reports results for systems with dynamics obstacles in the state space.

In all experiments involving obstacles, these dynamic unsafe regions are represented as Euclidean balls with a fixed radius. Further, to ensure consistent comparison across methods, a maximum verification time of 1000$s$ is imposed for each optimization instance. We note that for the time-invariant method, we present the results obtained for the highest barrier degree.
For a description of each of the benchmarks systems under consideration, we refer to 
\ifarxiv
the appendix.
\else
the appendix in \cite{mazouz2026stochastic}.
\fi
In what follows, we discuss the results.


The results show that our time-varying barrier certificates consistently yield higher lower-bound probabilities on $P_\safe$ compared to both time-invariant and interpolation-based approaches, across all benchmark systems and horizons.
In fact, empirical evaluations of $P_\safe$ show that time-varying barrier certificates obtain very tight lower-bounds. 
\ifarxiv
    The corresponding Monte Carlo simulations for some of the systems are reported in Fig.~\ref{fig:mc-verification}, where the observed empirical safety probabilities closely align with the certified bounds.
\fi
Overall, the results demonstrate that time-varying barrier certificates provide the most reliable and scalable framework for safety certification in stochastic systems, particularly as horizon length  and environmental complexity increase.

\textbf{Benchmark without obstacles. }
%
The results in Table~\ref{tab:barrier} highlight clear performance differences across the three approaches. The time-invariant  formulation generally requires high-degree polynomials to obtain non-conservative bounds, reflecting its limited ability to represent complex safety structure with a single global certificate.
This limitation becomes particularly evident as the horizon increases and in higher-dimensional or unstable systems, such as the oscillator and Dubin's car, where low safety bounds are obtained despite substantial computational effort. 
\ifarxiv
Although increasing the barrier degree somewhat alleviates this, the improvement comes at a substantial computational cost.
\else
Increasing the barrier degree somewhat alleviates this, at a substantial computational cost.
\fi

Interpolation-based certificates partially mitigate the conservatism, particularly in low-dimensional settings and for short horizons, where they achieve reasonably tight bounds at low computational cost. However, this improvement does not persist as the horizon increases or as system complexity grows. In several cases, such as the oscillator benchmark, their performance degrades noticeably with increasing horizon, with bounds collapsing to zero despite increases in computation time. Moreover, the method fail to scale to higher-dimensional systems; e.g., in the Dubin's car benchmark, the approach does not return a non-trivial solution.



In contrast, time-varying barrier certificates consistently produce tight lower bounds across all benchmarks. They maintain high safety probabilities even as the horizon increases, as seen in both linear and nonlinear dynamics, including cases where competing methods become either overly conservative or infeasible. This behavior indicates that the time-varying formulation effectively captures the temporal structure of the problem, enabling accurate approximation of the underlying safety probability step while remaining computationally tractable.

\ifarxiv
    \begin{figure*}[t]
    \centering
    \begin{subfigure}[t]{0.27\textwidth}
        \includegraphics[width=\linewidth]{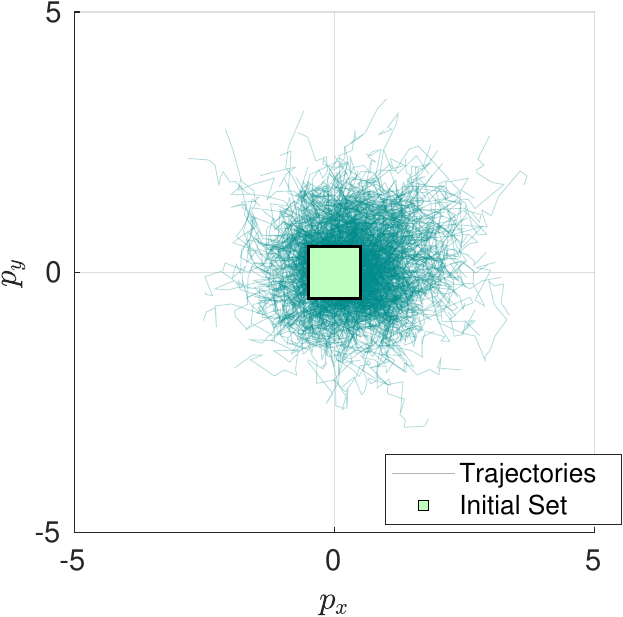}
        \caption{Dubin's Car: empirical $P_s = 0.99$}
        \label{fig:mc2}
    \end{subfigure}
    \hfill
    \begin{subfigure}[t]{0.33\textwidth}
        \includegraphics[width=\linewidth]{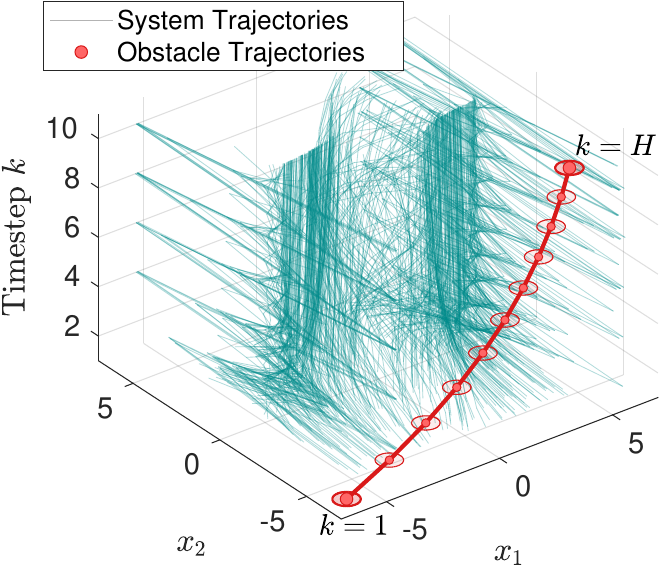}
        \caption{Oscillator: empirical $P_s = 0.87$}
        \label{fig:mc3}
    \end{subfigure}
    \hfill 
        \begin{subfigure}[t]{0.32\textwidth}
        \includegraphics[width=\linewidth]{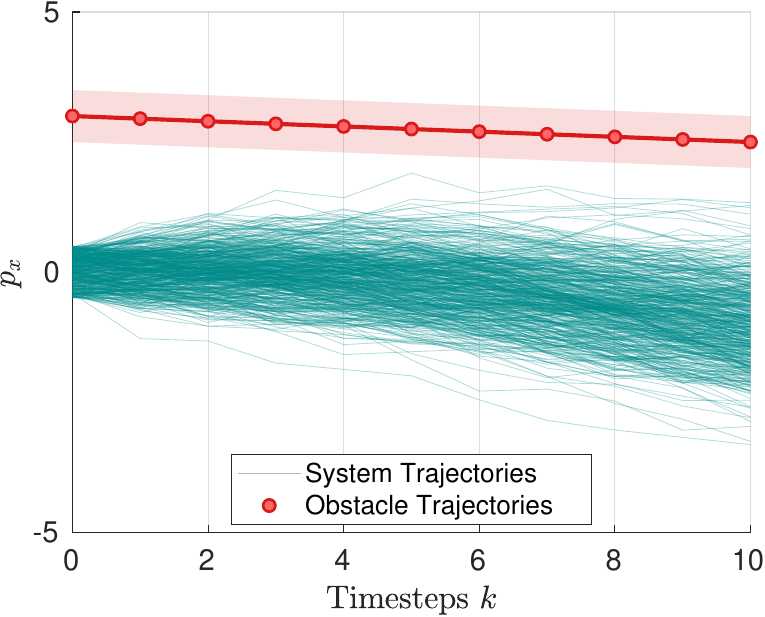}
        \caption{Quadrotor: empirical $P_s = 0.98$ }
        \label{fig:mc4}
    \end{subfigure}
        \caption{Monte Carlo simulations over a horizon ($H=10$) for the benchmark systems: dubin's car, oscillator, and quadrotor. The latter two incorporate dynamic obstacles, capturing time-varying interactions between the system and environment.}
    \label{fig:mc-verification}
\end{figure*}

\fi

\textbf{Benchmark with dynamic obstacles. } 
The results in Table~\ref{tab:barrier-dynamic-obtacles} further emphasize the importance of time-varying formulations when obstacle dynamics are embedded into the state space. This effect becomes particularly pronounced when multiple obstacles are present. Time-invariant barrier certificates again follow the same trend:  they require relatively high-degree polynomials and incur substantial computational cost, yet still yield conservative bounds that deteriorate as the horizon increases or as the obstacle dynamics introduce additional complexity. Similarly, interpolation-based methods again either produce conservative guarantees or fail to solve within the allotted time for more complex systems. 

Time-varying barrier certificates, however, consistently achieve significantly higher safety probabilities, while remaining tractable within the time limits. 
\ifarxiv
This trend is corroborated by the Monte Carlo results in Fig.~\ref{fig:mc-verification}, which confirm that the proposed time-varying certificates are not only less conservative but also predictive of actual system behavior. 
\fi
Collectively, these results indicate that explicitly modeling temporal structure is critical, especially when safety depends on reasoning over system and environment interaction.
\section{Conclusion}

This paper introduces a framework for safety certification of stochastic dynamical systems in environments with dynamic obstacles using stochastic barrier functions. The approach develops both time-invariant and time-varying barrier certificates for discrete-time systems subject to uncertainty, providing a formal lower bounds on the probability of safety. The formulation explicitly accounts for time-varying unsafe regions induced by obstacle dynamics, and, through a Bellman recursion perspective, the time-varying certificates are shown to capture temporal structure and yield  less conservative bounds than state-of-the-art approaches. By restricting the certificates to polynomial functions, the synthesis problem admits a convex sum-of-squares relaxation, enabling efficient solution via semidefinite programming. Through benchmark studies on nonlinear systems, the proposed framework is shown to achieve tight safety guarantees while maintaining favorable computational scalability.

\ifarxiv
    \section*{APPENDIX}
    \section*{Benchmark Systems}
\label{appendix}

\subsection{Unstable Systems}
To showcase the intuition of each framework, we consider two linear stochastic systems with unstable dynamics and additive noise \( \pw_k \sim \N(0, 0.01) \).
The dynamics of the both (one-dimensional) system are
\begin{align*}
    \px_{k+1} = 1.05 \px_k + \pw_k,
\end{align*}
with \( X_{\safe} = [-1.0, 1.0] \) and \( X_{\initial} = [-0.1, 0.1] \).

The second (two-dimensional) system has the same dynamics. It operates in 
\( X_{\safe} = [-1.0, 0.5]^2 \) and \( X_{\initial} = [\pm 0.05]^2 \). \\

For the 1D problem, we consider two cases:
\begin{itemize}
    \item The system evolves without any additional constraints.
    \item A dynamic obstacle is introduced as an interval of radius \( r = 0.2 \) with center \( o_k^{1} \in \mathbb{R} \), initialized at \( o_0^{1} = 0.8 \), and evolving according to $o^{1}_{k+1} = 1.02\, o^{1}_k.$
\end{itemize}

\subsection{Oscillator}
The Van der Pol oscillator from~\cite{abate2020arch} is modeled as a discrete-time stochastic system with polynomial dynamics 
%
\begin{align*}
    \px^{(1)}_{k+1} &= \px^{(1)}_k + \lambda \px^{(2)}_k + \pw^{(1)}, \\
    \px^{(2)}_{k+1} &= \px^{(2)}_k + \lambda\big(-\px^{(1)}_k + (1 - (\px^{(1)}_k)^2) \, \px^{(2)}_k\big) + \pw^{(2)},
\end{align*}
where \( \lambda= 0.1 \) is the sampling time and \( \pw_k \sim \N(0, 4 \cdot 10^{-4} I) \).
%
The state space is given by \( X = [-7, 7]^2 \), and the initial set is \( X_{\initial} = [-5, 5]^2 \).
We consider three cases:
\begin{itemize}
    \item The system evolves without any additional constraints.
    \item A dynamic obstacle is introduced as a ball of radius \( r = 0.5 \) with center \( o_k^{1} \in \mathbb{R}^2 \), initialized at \( o_0^{1} = (-6,-6) \), and evolving according to \( o^{1}_{k+1} = \begin{bmatrix} 0.85 & -0.15 \\ 0.02 & 1.00 \end{bmatrix} o^{1}_k \).
    \item A second obstacle with the same radius and dynamics is introduced, initialized at $o_0^{2} = (6, 6)$.
\end{itemize}

\subsection{Lotka-Volterra}

For the third benchmark, we consider a discrete-time Lotka-Volterra type stochastic system with state \( \px_k = (v_k, p_k) \in \mathbb{R}^2 \), where \( v \) and \( p \) denote the prey and predator populations, respectively. The dynamics are 
\begin{align*}
v_{k+1} &= v_k + T \left(\theta v_k (1 - v_k) - \phi v_k p_k \right) + \pw_1, \\
p_{k+1} &= p_k - T \left(\psi p_k - \delta v_k p_k \right) + \pw_2,
\end{align*}
where \( T = 0.1 \), \( \theta = 1.1 \), \( \phi = 0.4 \), \( \psi = 0.4 \), and \( \delta = 0.1 \). The noise satisfies \( \pw_k \sim \N(0, \mathrm{diag}(0.01,\,0.005)) \).

The state space is \( X = [0, 10] \times [0, 5] \), and the initial set is \( X_{\initial} = [6, 7] \times [2, 3] \).

\subsection{Dubin's Car}
For the next benchmark, we consider a Dubin's car model with a polynomial representation of the trigonometric dynamics. The state is given by \( \px_k = (p_x(k), p_y(k), \theta_k, v_k) \in \mathbb{R}^4 \), where \( (p_x, p_y) \) denote position, \( \theta \) is the heading angle, and \( v \) is the speed. The dynamics are given by
\begin{align*}
p_x(k+1) &= p_x(k) + T \, v_k \left(1 - \tfrac{1}{2}\theta_k^2 \right) + \pw_1, \\
p_y(k+1) &= p_y(k) + T \, v_k \left(\theta_k - \tfrac{1}{6}\theta_k^3 \right) + \pw_2, \\
\theta_{k+1} &= \theta_k + T \, \omega + \pw_3, \\
v_{k+1} &= v_k + \pw_4,
\end{align*}
where \( T = 0.05 \) is the sampling time, \( \omega = 0.8 \) is a constant turn rate, and the noise satisfies \( \pw_k \sim \N(0, \mathrm{diag}(0.05,\,0.05,\,0.02,\,0.02)) \).

The state space is \( X = [-5,5] \times [-5,5] \times [-1.5,1.5] \times [0,3] \), and the initial set is \( X_{\initial} = [-0.5,0.5] \times [-0.5,0.5] \times [-0.2,0.2] \times [0.5,1.0] \).

\subsection{Quadrotor}

Finally, we consider a planar quadrotor model with state \(x = (p_x, v_x, \theta, \omega) \in \mathbb{R}^4\), where \(p_x\) and \(v_x\) denote the horizontal position and velocity, and \(\theta, \omega\) denote the pitch angle and angular velocity, respectively. The polynomial discrete-time dynamics are given by
\begin{align*}
p_{x,k+1} &= p_{x,k} + T v_{x,k} + \pw_1, \\
v_{x,k+1} &= v_{x,k} + T \left(k_T\!\left(\theta_k - \tfrac{1}{6}\theta_k^3\right) - g\right) + \pw_2, \\
\theta_{k+1} &= \theta_k + T \omega_k + \pw_3, \\
\omega_{k+1} &= \omega_k + T \left(k_\tau \theta_k + \tfrac{1}{2}\theta_k^3\right) + \pw_4.
\end{align*}
where \( \pw_k \sim \N(0, \mathrm{diag}(0.05,\,0.1,\,0.02,\,0.05)) \).

The state space is defined as
$X = [-5,5] \times [-5,5] \times [-1.5,1.5] \times [-5,5],$
and the initial set is chosen near an upright configuration, i.e.,
$
X_0 = [-0.5,0.5] \times [-0.5,0.5] \times [-0.2,0.2] \times [-0.5,0.5].
$. 
We consider two cases:
\begin{itemize}
    \item The system evolves without any additional constraints.
    \item A dynamic obstacle is introduced as a ball of radius \( r = 0.5 \) with center \( o^{1}_k \in \mathbb{R}^2 \), initialized at \( o^{1}_0 = (3,-1) \), and evolving according to $o^{1}_{k+1} = \begin{bmatrix} 1 & 0.05 \\ 0 & 1 \end{bmatrix} o^{1}_k $. 
\end{itemize}



    

\fi

\bibliographystyle{IEEEtran}
\bibliography{bibliography}


\end{document}